\documentclass[journal]{IEEEtran}

\ifCLASSINFOpdf
\else
\fi

\usepackage{color}
\usepackage{bm}
\usepackage[ruled,vlined]{algorithm2e}
\usepackage{tikz} 
\usepackage{amssymb}
\usepackage{cite}
\usepackage{graphicx}
\usepackage{epsfig,subfigure}
\usepackage{amsmath}
\usepackage{algorithmic}
\usepackage{fancyhdr}
\usepackage{caption}
\usepackage{authblk}
\usepackage{ulem}

\usepackage{xcolor}
\usepackage[colorlinks = true,
            linkcolor = blue,
            urlcolor  = blue,
            citecolor = blue,
            anchorcolor = blue]{hyperref}

\begin{document}
\newtheorem{definition}{Definition}
\newtheorem{proposition}{Proposition}
\newcommand{\dbtilde}[1]{\accentset{\approx}{#1}}
\newcommand{\norm}[1]{\left\lVert#1\right\rVert}

\title{Enhancing Mixup-based Semi-Supervised Learning with Explicit Lipschitz Regularization}

\author[]{Prashnna Kumar Gyawali}
\author[]{Sandesh Ghimire}
\author[]{Linwei Wang}
\affil[]{pkg2182@rit.edu \\ Rochester Institute of Technology, NY, USA}

\maketitle

\begin{abstract}
The success of deep learning relies on the availability of large-scale annotated data sets, the acquisition of which can be costly requiring expert domain knowledge. Semi-supervised learning (SSL) mitigates this challenge by exploiting the behavior of the neural function on large unlabeled data.
The smoothness of the neural function 
is a commonly used assumption 
exploited in SSL. 
A successful example is 
the adoption of mixup strategy in SSL 
that enforces the global smoothness 
of the neural function 
by encouraging it to behave linearly when interpolating between training examples. 
Despite its empirical success, 
however, 
the  theoretical  underpinning  of  how mixup regularizes the neural function has not been fully understood. 
In this paper, we  
 offer  a  theoretically  substantiated 
 proposition that 
mixup improves the smoothness of the neural function by bounding the Lipschitz constant of the gradient function of the neural networks. 
We then propose that 
this can be strengthened by 
simultaneously constraining the
Lipschitz constant of the neural function itself through adversarial  Lipschitz  regularization, encouraging the neural function to  behave  linearly while also constraining  the  slope  of  this  linear function.  
On three benchmark data sets and one real-world biomedical data set, 
we demonstrate that 
this  combined  regularization 
results in improved generalization performance of SSL when learning from a small amount of labeled data. 
We further demonstrate the robustness of the presented method against single-step adversarial attacks. 
Our code is available at \href{https://github.com/Prasanna1991/Mixup-LR}{https://github.com/Prasanna1991/Mixup-LR}.
\end{abstract}


\IEEEpeerreviewmaketitle

\section{Introduction}
\IEEEPARstart{D}{eep} Learning has been an increasingly common choice of 
data analyses across various domains. They 
have achieved strong performance when trained with a large set of well-annotated data.  However, the acquisition of such data sets is expensive in many domains
as the annotation requires expert knowledge \cite{gyawali2019semi}. In comparison, the collection of a large amount of data without any annotation, \textit{i.e.,} unlabeled data set, is often less costly. This surplus of unlabeled data 
can be exploited to benefit the learning from small
labeled data 
via semi-supervised learning (SSL). 

Formally, in SSL, a data set $\mathcal{X} = \{ x_{1},x_{2},...,x_{n}\}$ is given among which only the first $m$ points are annotated $\{y_{1},...,y_{m}\} \in \mathcal{Y}$, and the remaining points are unlabeled. While learning the function $f: \mathcal{X}\rightarrow \mathcal{Y}$, 
SSL will exploit the hidden relationship within the data 
to predict the labels of unlabeled data points. 
An important assumption 
commonly exploited 
is the 
\textit{smoothness} of the neural function
\cite{zhou2004learning}.
Generally, this 
can be loosely categorized into two groups: 
\textit{local} smoothness, 
and 
\textit{global smoothness}
\cite{zhou2004learning}.
As illustrated in Fig. (\ref{fig:smoothSSL}) on
a classic two-moon toy problem,  
local assumption 
regularizes outputs of nearby points 
 to have the same label. 
 This is represented by 
 various perturbation-based methods that 
 constrain the output of the neural function in the vicinity of available data points
\cite{laine2016temporal, miyato2018virtual}, 
which 
does not consider the connection between data points. 
Alternatively, 
global assumption 
regularizes outputs of the points of the same structure (\textit{e.g.,} any points on a single moon in \ref{fig:smoothSSL}), 
more fully utilizing 
the information in the unlabeled data structures. 
This is represented by various graph-based methods \cite{zhou2004learning,luo2018smooth}, where the similarity of data points is defined by graph and outputs of neural functions are smoothed for the graph structure. 


\begin{figure}[t]
\begin{center}
\includegraphics[scale=0.30]{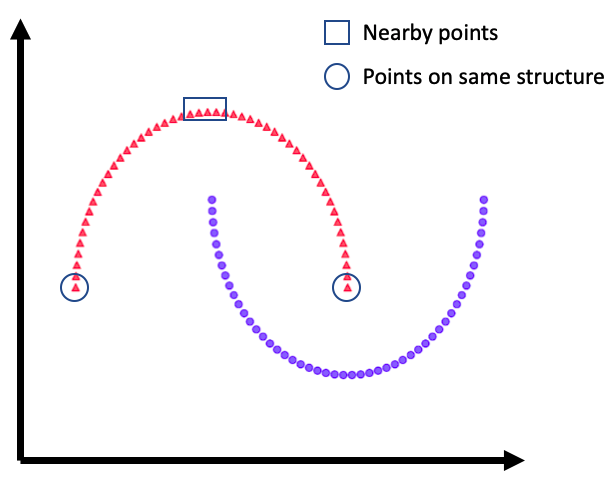}
\end{center}
\caption{The use of local and global smoothness in the classic two-moon toy SSL problem.} 
\label{fig:smoothSSL}
\end{figure}

More recently, the mixup regularizer \cite{zhang2017mixup}, initially proposed for supervised learning, has been applied to SSL and demonstrated state-of-the-art performance \cite{berthelot2019mixmatch, verma2019interpolation}.
Mixup trains a deep learning model on linear interpolants of inputs and labels. 
It has been considered as regularizing the global smoothness of the function by filling the void between input samples
and, in specific, has been interpreted 
to be 
encouraging 
a linear behavior of 
the underlying neural function when interpolating between training examples  \cite{zhang2017mixup}. 
Despite the empirical success obtained in many variants of the mixup strategy \cite{verma2018manifold,gyawali2020semi}, 
however, the theoretical underpinning of its regularization effect on the
the neural function has not been fully understood. 


Outside the regime of SSL, 
different learning theories for generalization have agreed that regularizing some notion of smoothness of the hypothesis class of the function helps improve generalization \cite{kawaguchi2017generalization}.
One increasingly popular approach 
to regularize the smoothness of the neural function, or to control the complexity of the function, 
is in
enforcing Lipschitz continuity of the deep network 
\cite{gouk2018regularisation, arjovsky2017wasserstein, asadi2018lipschitz}. 
This interest is particularly noticeable in the generative modeling community to improve the stability of GANs, where efficiently constraining the Lipschitz constant of the critic function is fundamental due to the nature of the underlying optimization problem (minimization of the Wasserstein distance between real and generated samples) \cite{arjovsky2017wasserstein}. 
The Lipschitz continuity of a function 
(see Definition \ref{lipschitz} for gradient function) 
essentially bounds the 
rate of changes in the function output as a result of the change in the input, preventing a function from changing steeply over its input space.
To enforce the Lipschitz continuity of a neural function (or to constrain the Lipschitz constant of a neural function),
several techniques have been proposed, such as weight clipping \cite{arjovsky2017wasserstein}, gradient penalty \cite{gulrajani2017improved}, and adversarial Lipschitz regularization \cite{terjek2019virtual}. 
Beside generative models, the Lipschitz continuity is also considered in several deep learning topics, including robust learning \cite{weng2018evaluating}, deep learning theory \cite{bartlett2017spectrally} and supervised learning \cite{ohn2019smooth}.
Despite these extensive efforts, the effect of Lipschitz continuity for SSL and 
its relation with
mixup-based regularization have received limited attention. 

\begin{figure}[t]
\begin{center}
\includegraphics[scale=0.18]{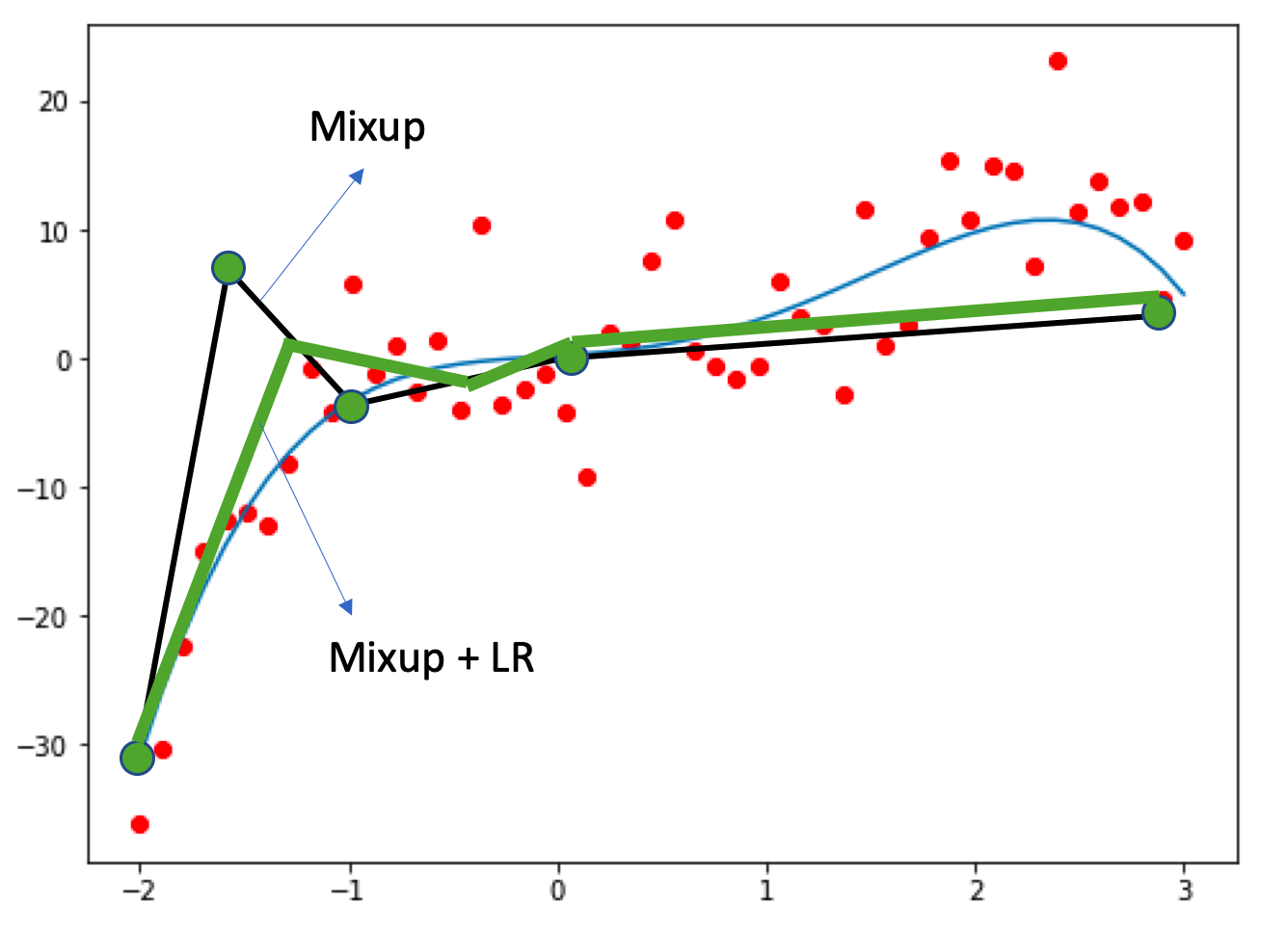}
\end{center}
\caption{Schematic of the mixup-based strategy for learning function (blue line) with limited labeled data (green dots). Mixup promotes linearity between the pair of data (black line) but does not bound the steepness of the slope. We add Lipschitz regularization to constrain the slope of such function (green line).}
\label{fig:mixup_lr}
\end{figure}

In this paper, 
we first offer a theoretically substantiated interpretation of 
the regularization effect of mixup from the lens of Lipschitz constant.
We show that, by promoting linearity, mixup minimizes the Lipschitz constant of the gradient function of the neural network, thereby enforcing the 
Lipschitz smoothness of the neural function 
(Proposition \ref{pro:1}). We then note that, 
while this minimizes the rate of change of the gradient of the neural function, it 
does not enforce the Lipschitz continuity of the neural function itself.  
Intuitively, 
this means that while the mixup encourages the function to behave linearly when interpolating, it does not bound or constrain the steepness of the slope of this linear function 
as illustrated in Fig. \ref{fig:mixup_lr}.
We therefore present a new SSL strategy that combines mix-up training with Lipshitz regularization (LR)
to simultaneously constraint the Lipschitz constant of both the neural function (through LR) and its gradient function (through mixup). 
This not only encourages the neural function to behave linearly, but also constrains the slope of this linear function as illustrated in Fig. \ref{fig:mixup_lr}. 
We hypothesize that this 
combined regularization will 
further smooth the neural function 
and 
result in improved generalization performance of SSL 
when learning from a small amount of labeled data. 


We test our hypothesis on three widely considered benchmark data sets (CIFAR-10 \cite{krizhevsky2009learning}, SVHN \cite{netzer2011reading}, and CIFAR-100 \cite{krizhevsky2009learning}) and a real world biomedical data set (Skin Lesion images \cite{codella2019skin,tschandl2018ham10000}).
We compare the performance of the presented method with standard SSL methods, including the state-of-the-art MixMatch \cite{berthelot2019mixmatch}, and demonstrate its improvement in generalization across all these datasets. 
We further investigate the robustness of the presented method to single-step adversarial attacks and demonstrate improved robustness.

In summary, the contribution of this work includes:
\begin{itemize}
    \item We establish the first connection between mixup-based regularization and the smoothness of the neural function via Lipschitz smoothness.
    \item We propose augmenting mixup-based approach with explicit Lipschitz regularization on the neural function to improve the generalization of SSL methods. 
    \item We demonstrated improved performance over the state-of-the-art methods in three benchmark datasets and a real-world biomedical dataset.
    \item We demonstrate improved robustness of the presented method to single-step adversarial attacks.
\end{itemize}

\section{Related Work}
Our work is generally related to two broad research topics of deep learning which we discuss separately below.
\subsection{Semi-supervised learning}
SSL has been extensively studied in machine learning. In recent times, with deep learning, SSL has seen tremendous success. One of the standard approaches of SSL algorithms is the application of consensus regularization. Toward this, perturbation around input data points (or its latent space) \cite{miyato2018virtual, laine2016temporal, sajjadi2016regularization, gyawali2019semi} have been considered. For instance, in $\Pi$-Model \cite{laine2016temporal, sajjadi2016regularization}, consistency-based regularization is applied on ensemble predictions obtained via techniques like random data augmentation, and network dropout.
On the other hand, virtual adversarial training (VAT) \cite{miyato2018virtual} maintains similar consistency by forcing predictions of different adversarially-perturbed inputs to be the same. However, by considering perturbations around single data points, these approaches regularize only the \textit{local} smoothness of the network function. Furthermore, it has been argued that such local perturbations would not fully utilize the information in the unlabeled data structure \cite{luo2018smooth}. 

For regularizing the global smoothness of the neural function, mostly graph-based methods have been considered \cite{zhou2004learning, luo2018smooth}.
These methods defined the similarity of the data points on the graph and smoothed outputs of the neural network for such graph structure. 
More recently, mixup regularizer \cite{zhang2017mixup} is argued to regularize the global smoothness of the neural function and, thus, help achieve generalization in SSL \cite{gyawali2020semi}.
Initially, mixup was proposed for improving generalization in supervised learning, demonstrating state-of-the-art results in the corresponding benchmarks \cite{zhang2017mixup}.
Later, mixup was extended to the SSL problems, such as MixMatch \cite{berthelot2019mixmatch}, where a even larger margin of improvements were obtained across many datasets.
Different variants of mixup have since been presented in the literature, such as mixing in both data and latent space for further improving generalization in supervised learning \cite{verma2018manifold} and SSL \cite{gyawali2020semi}. 
However, despite such empirical success, the theoretical underpinning of how mixup regularizes the neural function has not been fully understood.

Motivated by the performance of a mixup-based strategy, in this paper, we 
offer a theoretical insight of the regularization effect of mixup through the lens of Lipschitz constant and, based on which, identify a complementary improvement to improve mixup-based SSL.


\subsection{Lipschitz regularization}
The Lipschitz constant of the network has been proposed as a candidate measure for the Rademacher complexity (a measure of generalization) \cite{bartlett2017spectrally}.
Different Lipschitz regularizations for controlling Lipschitz constant have been considered in various topics of machine learning.
Lipschitz continuity is commonly considered for robust learning, avoiding adversarial attacks \cite{weng2018evaluating}, and stabilization to train generative adversarial networks \cite{arjovsky2017wasserstein}. 
To enforce the Lipschitz constraint, different implicit approaches like weight clipping \cite{arjovsky2017wasserstein} and gradient penalty \cite{gulrajani2017improved,wei2018improving} have been considered. 
These implicit approaches approximate the constraint on Lipschitz constant by, typically, 
penalizing the norm of the function gradient at certain input points \cite{gulrajani2017improved}.
An explicit approach to 
Lipschitz regularization, on the other hand, attempts to directly encourages Lipschitz continuity based on its definition, 
which was argued to provides more control over the regularization effect
\cite{terjek2019virtual}. 
In \cite{terjek2019virtual}, 
for instance, 
this was done by explicitly penalizing the violation of Lipschitz constraint.
Despite these extensive studies, 
however, the use of Lipschitz regularization 
has not been considered for improving generalization of SSL methods.

\section{Preliminaries}
In this section, we first briefly discuss some preliminaries required for the presented method in Section \ref{sec:method}.

\subsection{Preliminary I: Mixup Regularization}
Mixup \cite{zhang2017mixup} is a data-dependent regularization inspired by Vicinal Risk Minimization (VRM) principle \cite{chapelle2001vicinal} that encourages the model $f$ to behave linearly in-between training samples. Formally, the mixup produces virtual feature-target vector:
\begin{align}
    \nonumber
    \Tilde{x} & = \lambda x_{1} + (1 - \lambda) x_{2}, \\
    \nonumber
    \Tilde{y} & = \lambda y_{1} + (1 - \lambda) y_{2}
\end{align}
where ($x_{1}$, $y_{1}$) and ($x_{2}$, $y_{2}$) are two feature-target vectors drawn randomly from the training data and $\lambda \sim \text{Beta}(\alpha, \alpha) \in [0,1]$. 
This mixup is used to construct a virtual dataset $\mathcal{D}_{v}$ := $\{(\Tilde{x_{i}}, \Tilde{y_{i}})\}^{m}_{i=1}$ which is then used to train the network function $f$ by minimizing the loss value:
\begin{align}
\label{eq:basic_mixup}
    \mathcal{L}_{mixup} =  \frac{1}{m}\sum_{i}^{m} \ell (f(\Tilde{x_{i}}, \Tilde{y_{i}}))
\end{align}
In between the original feature-target pairs, the loss function $\mathcal{L} (f)$ encourages the network function $f$ to behave linearly:
\begin{align}
    f(\lambda x_{1} + (1 - \lambda) x_{2}) = \lambda f(x_{1}) + (1-\lambda)f(x_{2})
\end{align}


\subsection{Preliminary II: Lipschitz Regularization}
\label{sec:lip}
A general definition of the smallest Lipschitz constant K of a function $f:\mathcal{X}\rightarrow\mathcal{Y}$ is:
\begin{equation}
    \begin{aligned}
    \norm{f}_{K} = \underset{x_{1}, x_{2}\in\mathcal{X};x_{1}\neq x_{2}}{\text{sup}} \frac{d_{Y}(f(x_{1}),f(x_{2}))}{d_{X}(x_{1}, x_{2})}
\end{aligned}
\end{equation}
where the metric spaces $d_{X}$ and $d_{Y}$ are the domain and co-domain of the function $f$, respectively. The properties of low Lipschitz constants for deep networks are explored in \cite{oberman2018lipschitz}, demonstrating that it improves generalization. Recent literature in stabalizing GAN have proposed different approaches for Lipschitz regularization.
These regularizations can be generally grouped into an implicit and explicit form of penalization for violation of Lipschitz constraint. 
For instance, gradient penalty \cite{gulrajani2017improved}, an implicit Lipschitz regularization, penalizes the norm of the gradient as
\begin{equation}
    \begin{aligned}
\label{eq:gp_}
\mathbb{E}_{x\sim\mathcal{X}}(\norm{\triangledown_{x}f(x)}_{2} - 1)^{2}
\end{aligned}
\end{equation}
and Lipschitz penalty \cite{petzka2017regularization, terjek2019virtual}, an explicit regularization, penalizes the violation of the Lipschitz constraint as

\begin{equation}
    \begin{aligned}
\label{eq:gp}
\mathbb{E}_{x_{1}, x_{2}\sim\mathcal{X}}(\frac{|f(x_{1}) - f(x_{2})|}{\norm{x_{1} - x_{2}}_{2}} - 1)^{2}  
\end{aligned}
\end{equation}
In both case, the objective is to achieve 1-Lipschitz function for $f$, an optimization requirement for Wassterstein-GANs.  

\section{Methodology}
\label{sec:method}
We consider a mapping function $f: \mathcal{X}\rightarrow \mathcal{Y}$, approximated via a deep network.
As has been empirically shown, the given function $f$ achieves better generalization when trained with a mixup strategy \cite{zhang2017mixup}. In \ref{sec:mixup}, we first establish the theoretical connection between mixup and Lipschitz regularization (of the gradient function of the neural network). With this understanding, in \ref{sec:lipschitzReg}, we propose a complementary improvement of mixup by explicit Lipschitz regularization of the neural function via adversarial Lipschitz regularization. 
Finally, in \ref{sec:model}, we integrate these ideas into a new SSL method.

\subsection{Mixup bounds the Lipschitz constant of the gradient of the neural function}
\label{sec:mixup}
To understand the role of mixup in regularizing the 
\textit{smoothness} of the neural function, 
we consider the definition of Lipschitz smoothness. 
\begin{definition}{(Lipschitz Smoothness).}
\label{lipschitz}
A differential function $f:\mathcal{X}\rightarrow\mathcal{Y}$
is Lipschitz smooth with constant $L > 0$ if its derivatives are Lipschitz continuous:
\begin{align*}
   \norm{\bigtriangledown f(x_{1}) - \bigtriangledown f(x_{2})} \leq L \norm{x_{1} - x_{2}}, \quad  \forall x_{1}, x_{2} \in \mathcal{X}.
\end{align*}
\end{definition}
Loosely speaking, when the gradient of a function is Lipschitz continuous, such function are considered to be smooth. 

We now show that mixup regularization is a lower bound of the Lipschitz constant of the gradient of the neural network.
\begin{proposition}
\label{pro:1}
Let $f$ be the differential function with a Lipschitz continuous gradient over $\mathbb{R}^n$ with constant $L$. This function, $f$, via mixup regularizer \cite{zhang2017mixup}, is encouraged toward convexity such that $f(\lambda x_{1} + (1 - \lambda) x_{2}) = \lambda f(x_{1}) + (1-\lambda)f(x_{2})$ with $\lambda \in [0,1]$. 
Then, $    f(\alpha x_{1} + (1 - \alpha) x_{2}) - (\alpha f(x_{1}) + (1-\alpha)f(x_{2})) \leq \frac{\alpha (1 - \alpha) L}{2} \norm{x_{1} - x_{2}}^2$.

\end{proposition}
\begin{proof}
With the definition of Lipschitz smoothness, for a differential function $f:\mathcal{X}\rightarrow\mathbb{R}$, $\mathcal{X} \subset \mathbb{R}^{n}$, and a constant $L$, we have,
\begin{align}
       \norm{\bigtriangledown f(x_{1}) - \bigtriangledown f(x_{2})} \leq L \norm{x_{1} - x_{2}}, \quad  \forall x_{1}, x_{2} \in \mathcal{X}.
\end{align}
Note that this definition does not assume convexity of $f$. But when we assume that the function $f$ is convex, then using Cauchy-Schwartz inequality, we have equivalent condition as: 
\begin{align}
\label{eq:cauchy}
       (\bigtriangledown f(x_{1}) - \bigtriangledown
       f(x_{2}))^{T}(x_{1}-x_{2}) \leq L \norm{x_{1} - x_{2}}^{2}, \notag\\ \quad  \forall x_{1}, x_{2} \in \mathcal{X}.
\end{align}
Similarly, for the convex function $f$, using monotonicity of gradient equivalence, we have following condition:
\begin{align}
\label{eq:mono}
       (\bigtriangledown f(x_{1}) - \bigtriangledown f(x_{2}))^{T}(x_{1}-x_{2}) \geq 0, \quad  \forall x_{1}, x_{2} \in \mathcal{X}.
\end{align}
Now, let us consider the function
\begin{align}
\label{eq:g}
    g(x) = \frac{L}{2}x^{T}x - f(x)
\end{align}
Using (\ref{eq:cauchy}) and (\ref{eq:mono}), we first establish that $g$ is convex. We apply $x_{1}$ and $x_{2}$ to $g$, take the derivative, and subtract the two result to get:
\begin{equation}
\label{eq:gToConvex}
    \begin{aligned}
     \bigtriangledown g(x_{1}) & - \bigtriangledown g(x_{2})  = Lx_{1} -  \bigtriangledown f(x_{1}) - Lx_{2} + \bigtriangledown f(x_{2}) \\
     \bigtriangledown g(x_{1}) & - \bigtriangledown g(x_{2})   = (Lx_{1} - Lx_{2} - (\bigtriangledown f(x_{1}) - \bigtriangledown f(x_{2}))) \\
     (\bigtriangledown g(x_{1}) & - \bigtriangledown g(x_{2}))^{T}(x_{1} - x_{2}) \\  & = (L(x_{1} - x_{2})  - (\bigtriangledown f(x_{1}) - \bigtriangledown f(x_{2})))^{T}(x_{1} - x_{2}) \\
     & = L \norm{x_{1} - x_{2}}^{2}) - (\bigtriangledown f(x_{1}) - \bigtriangledown f(x_{2}))^{T}(x_{1} - x_{2}) \\
     & \geq 0
\end{aligned}
\end{equation}
This is equivalent to the convex form of (\ref{eq:mono}), and hence $g$ is convex. Now, when we expand $g$ according to the \textit{standard} definition of convexity, we get:
\begin{equation}
\label{eq:final}
    \begin{aligned}
    f(\alpha x_{1} + (1 - \alpha) x_{2}) - & (\alpha f(x_{1}) + (1-\alpha)f(x_{2})) \\ & \leq \frac{\alpha (1 - \alpha) L}{2} \norm{x_{1} - x_{2}}^2
\end{aligned}
\end{equation}
where the LHS of (\ref{eq:final}) is the minimization of mixup loss and we finished the proof.
\end{proof}

Note that only the constant $L$ on the right-hand side is subject to change during optimization for a given pair of data $x_{1}$ and $x_{2}$. 
As such, this proposition implies that minimizing mixup loss controls the constant $L$ which can be considered as the Lipschitz constant of the gradient function of 
the neural network, thereby making the function smoother. 
Establishing this theoretically-substantiated connection between mixup and Lipschitz regularization is the first contribution of this work.


\subsection{Bounding the Lipschitz constant of the neural function by adversarial Lipschitz regularization}
\label{sec:lipschitzReg}

As illustrated in Fig. \ref{fig:mixup_lr}
when regularizing the neural function to interpolate linearly with mixup provides the smoothest possible function among possible choices, it does not 
constrain the steepness of the slope of the linear function. Therefore,
in this section, we propose to 
augment mixup strategy with 
an explicit Lipschitz regularization to bound the Lipschitz constant of the neural function itself. 
In specific, 
we consider 
penalizing the violation of Lipschitz constraint as:
\begin{equation}
    \begin{aligned}
    \label{eq:lp}
    \mathcal{L}_{LP} =  \Big(\frac{d_{Y}(f(x_{1}),f(x_{2}))}{d_{X}(x_{1}, x_{2})} - \gamma)
\end{aligned}
\end{equation}
where $d_{X}$ and $d_{Y}$ are the metric for input and output space, respectively.
Note that we put $\gamma$ = 0 because, unlike the GAN setup, we are not required to obtain 1-Lipschitz function, as shown in Eq. (\ref{eq:gp}).

The direct implementation of Eq.(\ref{eq:lp}) is not trivial, 
partly due to the sampling strategy of 
the training pairs of $x_{1}$ and $x_{2}$. 
We follow the adversarial Lipschitz regularization strategy presented in \cite{terjek2019virtual} 
that penalizes Eq.(\ref{eq:lp}) 
on a pair of data points that 
maximizes the Lipschitz ratio. 
In specific, we first select the data point $x_{2}$ to be in the vicinity of the training point $x_{1}$ such that $x_{2}$ = $x_{1} + r$:
\begin{equation}
    \begin{aligned}
    \norm{f}_{K} = \underset{x_{1}, x_{1}+r\in\mathcal{X};x_{1}\neq x_{2}}{\text{sup}} \frac{d_{Y}(f(x_{1}),f(x_{1}+r))}{d_{X}(x_{1}, x_{1}+r)}
\end{aligned}
\end{equation}
where the mapping $f$ is $K$-Lipschitz if taking maximum over $r$ results in value $K$ or smaller.
Toward this, we define $r$ by finding the adversarial perturbation that maximizes the Lipschitz ratio for the given $x_{1}$ as:
\begin{equation}
    \begin{aligned}
    r_{adv} = \underset{x_{1}, x_{1}+r\in\mathcal{X};x_{1}\neq x_{2}}{\text{arg max}} \frac{d_{Y}(f(x_{1}),f(x_{1}+r))}{d_{X}(x_{1}, x_{1}+r)}
\end{aligned}
\end{equation}
and penalize the corresponding maximum violation of the Lipschitz constraint  as:
\begin{equation}
    \begin{aligned}
    \label{eq:alp}
    \mathcal{L}_{ALP} =  \Big(\frac{d_{Y}(f(x_{1}),f(x_{1} + r_{adv}))}{d_{X}(x_{1}, x_{1} + r_{adv})} - \gamma)
\end{aligned}
\end{equation}
Since computing adversarial perturbation is a nonlinear optimization problem, we followed a crude and cheap power iteration based approximation approach similar to works in \cite{miyato2018virtual, terjek2019virtual}.
In this iterative scheme, we approximate the direction at $x_{1}$ that induces the largest change in the output in terms of divergence $d_{Y}$.   

Combining the mixup loss of Eq. (\ref{eq:basic_mixup}) and ALR of Eq. (\ref{eq:alp}), we obtain our proposed loss function:
\begin{equation}
    \begin{aligned}
    \label{eq:combined}
    \mathcal{L} = \mathcal{L}_{mixup} + \zeta \cdot \mathcal{L}_{ALP}
\end{aligned}
\end{equation}
where $\zeta$ controls the relative strength of explicit Lipschitz regularization. 

\subsection{Integrating mixup and explicit Lipschitz regularization for SSL}
\label{sec:model}
In this section, we apply the presented combination of loss (Eq. \ref{eq:combined}) for SSL setup. Toward this, we consider a data set $\mathcal{X} = \{ x_{1},x_{2},...,x_{n}\}$ among which only the $m$ points are annotated with labels $\{y_{1},...,y_{m}\} \in \mathcal{L}$, and the remaining points are unlabeled. We aim to learn parameters $\theta$ for the mapping function $f: \mathcal{X}\rightarrow \mathcal{Y}$, approximated via a deep network. 

Similar to MixMatch algorithm \cite{bartlett2017spectrally}, along the course of training, we first \textit{guess} and continuously update the labels for unlabeled data points. We augment $P$ separate copies of unlabeled data batch $u_{b}$, and compute the average of the model's prediction as:
\begin{equation}
  \label{eq:labelEstimation}
q_b = \frac{1}{P}\sum_{p=1}^{P} f(u_{b, p}; \theta)
\end{equation}
Note that the label guessing in this manner also regularize the model toward consistency 
in a similar fashion to 
typical perturbation-based approaches, 
as the data transformations (\textit{e.g.}, rotation, translation, etc.) are assumed to leave class semantics unaffected. 

While generating a guessed label, we sharpen the obtained labels to minimize the entropy in our estimation. Entropy minimization is a traditional and successful strategy in the SSL to enforce the classifier output to have low-entropy predictions on unlabeled data \cite{berthelot2019mixmatch, miyato2018virtual}. For the sharpening function, we use the following operation:
\begin{equation}
  \label{eq:sharpening}
  \text{Sharpen} (q_b,\tau)_{i} := q_{b_{i}}^{\frac{1}{\tau}} / \sum_{j=1}^{S} q_{b_{j}}^{\frac{1}{\tau}}
\end{equation}
where $S$ represents the number of classes in the output space, and $\tau$ is the temperature hyperparameter for the categorical distribution. However, note that such sharpening is not feasible in a multi-label classification scenario \cite{gyawali2020semi}. 

We then use this sharpened guessed labels for unlabeled data points, and ground truth labels for labeled data points to train the network using mixup strategy. In each batch, we mix both labeled and unlabeled data points together to ensure that the mixed data fairly represent the distribution of both labeled and unlabeled data. 
\begin{equation}
  \label{eq:mixMatch}
  \begin{aligned}
  \lambda & \sim \text{Beta}(\alpha, \alpha)\\
  \lambda' & = \text{max}(\lambda, 1-\lambda) \\
  \Tilde{x} & = \lambda' \mathbf{x}_{1} + (1 - \lambda') \mathbf{x}_{1} \\
  \Tilde{y} &= \lambda' \mathbf{y}_{2} + (1 - \lambda') \mathbf{y}_{2}
  \end{aligned}
\end{equation} 

It is reasonable to expect that the actual labels in labeled data are more reliable than guessed labels for unlabeled data, which 
motivates us to use different loss functions for labeled and unlabeled data points. Since we mixed them together, we use $\lambda' = \text{max}(\lambda, 1-\lambda)$ in Eq.(\ref{eq:mixMatch}) to ensure $\Tilde{x}$ is closer to $x_{1}$ than $x_{2}$: this knowledge then allow us to apply labeled and unlabeled loss according to the index of $x_{1}$. 
For data points in a batch $\mathcal{B}$  
that are closer to labeled data, 
we apply following supervised loss term:
\begin{equation}
  \label{eq:xloss}
  \begin{aligned}
\mathcal{L}_{\mathcal{X}} = 
\underset{\mathcal{B}}{\sum} \ell (f(\Tilde{x}), \Tilde{y})
  \end{aligned}
\end{equation} 
For data points in $\mathcal{B}$  
that are closer to unlabeled data, 
we apply $L_{2}$ loss as it is considered to be less sensitive to 
incorrect predictions:
\begin{equation}
  \label{eq:uloss}
  \begin{aligned}
\mathcal{L}_{\mathcal{U}} = 
\underset{\mathcal{B}}{\sum}  \norm{f(\Tilde{x}) - \Tilde{y}}_{2}^{2} 
  \end{aligned}
\end{equation} 
Finally, we combined ALP loss as defined in Eq. (\ref{eq:alp}) with the mixup-based loss terms of Eq. (\ref{eq:xloss}) and Eq. (\ref{eq:uloss}) as:
\begin{equation}
  \label{eq:final_loss}
  \begin{aligned}
\mathcal{L} = 
\underset{\text{mixup loss}}{\underbrace{\mathcal{L}_{\mathcal{X}} + \lambda \cdot \mathcal{L}_{\mathcal{U}}}} + \underset{\text{Lipschitz regularization}}{\underbrace{\zeta \cdot \mathcal{L}_{ALP}}}
  \end{aligned}
\end{equation}
where $\lambda_{\mathcal{U}}$ is the weight term for the unsupervised loss, and $\zeta$ is the weight term for the explicit Lipschitz regularization presented in \ref{sec:lipschitzReg}. We refer to the model trained in this manner as Mixup-LR throughout the rest of the manuscript for brevity.




\section{Experiments}
We test the effectiveness of the presented Mixup-LR on three standard SSL benchmark datasets (CIFAR-10 \cite{krizhevsky2009learning}, SVHN \cite{netzer2011reading} and CIFAR-100 \cite{krizhevsky2009learning}) and a real-world biomedical dataset (Skin Lesion images \cite{codella2019skin,tschandl2018ham10000}). We also consider the robustness of the presented Mixup-LR against adversarial attacks (section \ref{sec:robustness}). 

\subsection{Implementation details}
In all standard SSL benchmark experiments, we use the Wide ResNet-28 model from \cite{oliver2018realistic}, and for the biomedical dataset, we use the AlexNet model from \cite{gyawali2020semi}. Our implementation of the model and training procedure closely matches that of \cite{berthelot2019mixmatch}. For benchmark data sets, we follow modern standards in SSL and report the median error rate of the last 20 checkpoints on all the unlabeled data points, and on the biomedical dataset, we follow the classic approach and report the result on test data by choosing the checkpoint with the lowest validation error. In all experiments, we linearly ramp up $\lambda$ to its maximum value over the training steps.
We set $\zeta$ hyperparameter to 2 in all the cases, and consider only 1 iteration to calculate adversarial perturbation $r_{adv}$. 
Given the diversity of data sets considered, we leave other specific implementation details to each subsection. 

For comparison, we consider three existing SSL methods from \cite{berthelot2019mixmatch}: $\Pi$-Model \cite{laine2016temporal, sajjadi2016regularization}, Virtual Adversarial Training \cite{miyato2018virtual}, and MixMatch \cite{berthelot2019mixmatch}. The first two methods represent SSL considering local smoothness, and MixMatch represents SSL with global smoothness via a mixup strategy. Since MixMatch inspires the presented Mixup-LR, we re-implemented MixMatch in the same codebase to ensure a fair comparison. Furthermore, we trained each model on five random seeds and reported the mean and standard deviation of the error rates. 

\begin{table}[t]
\centering
    \begin{tabular}[t]{l|ccc}
    \hline
    Methods/Labels & 250 & 1000 & 4000\\
    \hline
    $\Pi$-Model \cite{berthelot2019mixmatch} & 53.02 $\pm$ 2.05 & 31.53 $\pm$ 0.09 & 17.41 $\pm$ 0.37 \\
    VAT \cite{berthelot2019mixmatch} & 36.03 $\pm$ 2.82 & 18.68 $\pm$ 0.40 & 11.05 $\pm$ 0.31 \\
    MixMatch \cite{berthelot2019mixmatch} & 11.08 $\pm$ 0.87 &7.75 $\pm$ 0.32 & 6.24 $\pm$ 0.06 \\
    MixMatch (ours) & 12.90 $\pm$ 2.10 & 8.73 $\pm$ 0.29 & 6.29 $\pm$ 0.11\\
    Mixup-LR & \textbf{9.47} $\pm$ 0.99 & \textbf{7.59} $\pm$ 1.64& \textbf{5.44} $\pm$ 0.06\\
    \hline
    \end{tabular} 
    \caption{Error rate (\%) comparison of Mixup-LR to baseline methods on CIFAR10 for a varying number of labels.
    }
\label{tab:cifar10}
\end{table}

\begin{table}[t]
\centering
    \begin{tabular}[t]{l|ccc}
    \hline
    Methods/Labels & 250 & 1000 & 4000\\
    \hline
    $\Pi$-Model \cite{berthelot2019mixmatch} & 17.65 $\pm$ 0.27 & 8.60 $\pm$ 0.18 & 5.57 $\pm$ 0.14 \\
    VAT \cite{berthelot2019mixmatch} & 8.41 $\pm$ 1.01 & 5.98 $\pm$ 0.21 & 4.20 $\pm$ 0.15 \\
    MixMatch \cite{berthelot2019mixmatch} & 3.78 $\pm$ 0.26 & 3.27 $\pm$ 0.31 & 2.89 $\pm$ 0.06 \\
    MixMatch (ours) & 3.65 $\pm$ 0.25 & 3.26 $\pm$ 0.13 & 2.87 $\pm$ 0.05\\
    Mixup-LR & \textbf{3.58} $\pm$ 0.30 & \textbf{3.09} $\pm$ 0.13 & \textbf{2.81} $\pm$ 0.04\\
    \hline
    \end{tabular} 
    \caption{Error rate (\%) comparison of Mixup-LR to baseline methods on SVHN for a varying number of labels.
    }
\label{tab:svhn}
\end{table}

\subsection{Generalization performance of SSL}
\label{sec:ssl}

\subsubsection{CIFAR-10}
CIFAR-10 is the standard SSL benchmark datasets with 60000 data points divided uniformly across ten labels. We evaluate the accuracy of each method considered with a varying number of labeled examples (250, 1000, and 4000) on all the unlabeled data sets. This means for the case of a labeled number of 250, we report the performance on the rest of the 59750 unlabeled samples. The $\lambda$ value was set to 75. We present the results on Table \ref{tab:cifar10}. As compared with the perturbation-based approach ($\Pi$-Model and VAT), the mixup strategy (via MixMatch) achieves better generalization throughout all the cases. By augmenting the mixup strategy with explicit Lipschitz regularization, the presented Mixup-LR further improves the generalization performance 
($p<0.02$ for $l=250$ and $l=4000$, unpaired \textit{t}-test). For instance, for the case of labeled number of 250, the presented Mixup-LR reduces the mean error rate by nearly 25\%. 

\subsubsection{SVHN}
SVHN consists of 73257 samples divided across ten labels. Similar to CIFAR-10, we evaluate the accuracy of each method considered  with varying numbers of labeled examples (250, 1000, and 4000). The $\lambda$ value was set to 250. The obtained results are presented in Table \ref{tab:svhn}. Similar to CIFAR-10, the presented Mixup-LR achieves better generalization across all the labeled training setup compared to both the perturbation-based approaches ($\Pi$-Model and VAT) and mixup-based approach (MixMatch, $p = 0.04$ for $l=4000$, $p\approx 0.20$ for $l=250$ and $l=1000$, unpaired \textit{t}-test). 

\begin{table}[t]
\centering
    \begin{tabular}[t]{l|c|c}
    \hline
    Methods/Labels & 10000 & 15000\\
    \hline
    MixMatch & 31.67 $\pm$ 0.30 & 28.50 $\pm$ 0.11 \\
    Mixup-LR & \textbf{29.11} $\pm$ 0.10 &  \textbf{26.21} $\pm$ 0.15\\
    \hline
    \end{tabular} 
    \caption{Error rate (\%) comparison of Mixup-LR to MixMatch on CIFAR100 for a varying number of labels.
    }
\label{tab:cifar100}
\end{table}

\begin{table}[t]
\centering
    \begin{tabular}[t]{l|c|c}
    \hline
    Methods/Labels & 600 & 1200\\
    \hline
    Supervised baseline & 0.80 $\pm$ 1.70 & 0.85 $\pm$ 1.06 \\
    MixMatch & 0.88 $\pm$ 1.21 & 0.89 $\pm$ 0.47 \\
    Mixup-LR & \textbf{0.89} $\pm$ 0.73 & \textbf{0.90} $\pm$ 0.66 \\
    \hline
    \end{tabular} 
    \caption{AUROC comparison of Mixup-LR to MixMatch on Skin Lesion dataset for a varying number of labels.
    }
\label{tab:skin}
\end{table}

\subsubsection{CIFAR-100}
CIFAR-100 is similar to CIFAR-10 except that it has 100 classes containing 600 images each. The $\lambda$ value was set to 250. Note that due to the increased complexity of the dataset, some prior works \cite{berthelot2019mixmatch} have suggested the use of a larger model (26 million parameter) instead of the base model considered in this work (1.5 million parameter). As such, the presented results might confound with results previously reported in the literature. Thus, we have only considered the evaluation of the baseline MixMatch and presented Mixup-LR on two different numbers of labeled examples (10000 and 15000). As shown in Table \ref{tab:cifar100}, the presented Mixup-LR significantly improved the performance compared to MixMatch in both cases ($p<0.0001$, unpaired \textit{t}-test)): the mean error rate for MixMatch was reduced by around 9\% with the presented Mixup-LR. 

\subsubsection{Skin Lesion}
ISIC 2018 skin data set comprises of 10015 dermoscopic images with labels for seven different disease categories. To evaluate the presented method, we created two sets of labeled training data (600 and 1200) considering the class balance. Similar to CIFAR-100, we compared the presented Mixup-LR against the baseline MixMatch. In this case, the $\lambda$ value was set to 50. Unlike other cases, to maintain the standards of the dataset, we report the AUROC. 
Since this biomedical dataset is relatively less considered in the SSL literature, we also included the supervised baseline. The obtained results are shown in Table \ref{tab:skin}, where, similar to other datasets, the presented Mixup-ALR achieves better generalization compared to the MixMatch approach, and the supervised baseline. 
The receiver operating characteristic (ROC) curves for the corresponding labeled number of 600 and 1200 are respectively presented in Fig. \ref{fig:roc_1} and Fig. \ref{fig:roc_2}. In these figures, we randomly selected the model for demonstrating the ROC comparison among the five random seed models.

\begin{figure}[t]
\begin{center}
\includegraphics[scale=0.35]{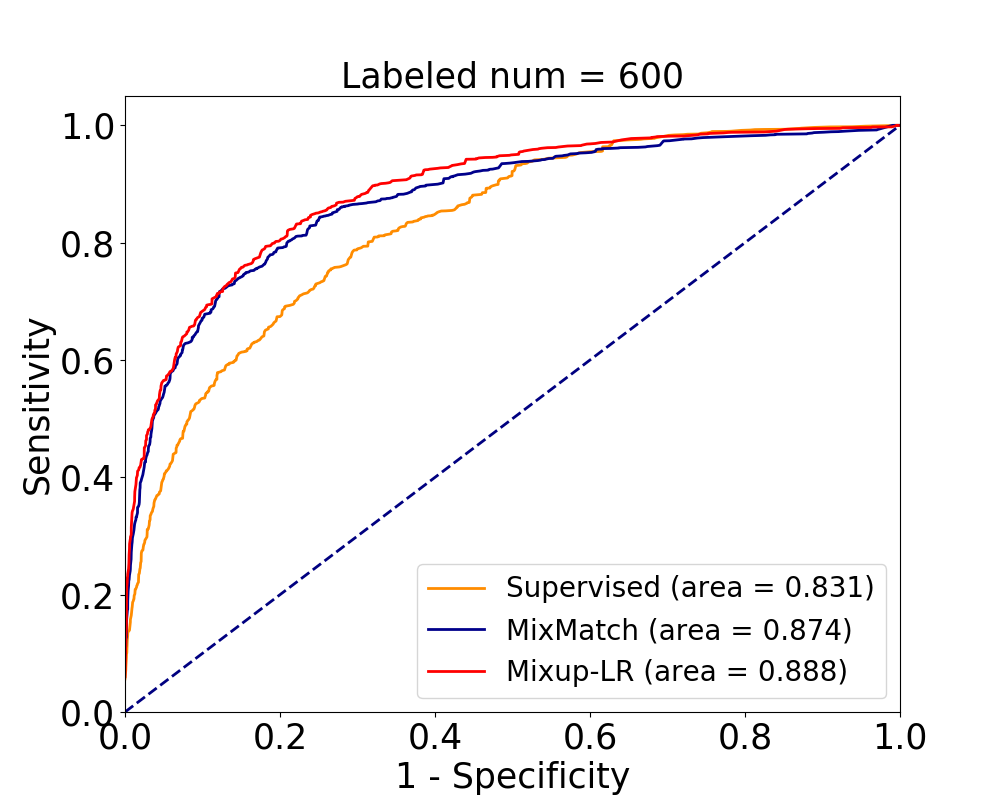}
\end{center}
\caption{ROC curves of the presented Mixup-LR compared to alternative models for classification of dermoscopic images into seven different disease categories when trained with 600 labels.} 
\label{fig:roc_1}
\end{figure}

\begin{figure}[t]
\begin{center}
\includegraphics[scale=0.35]{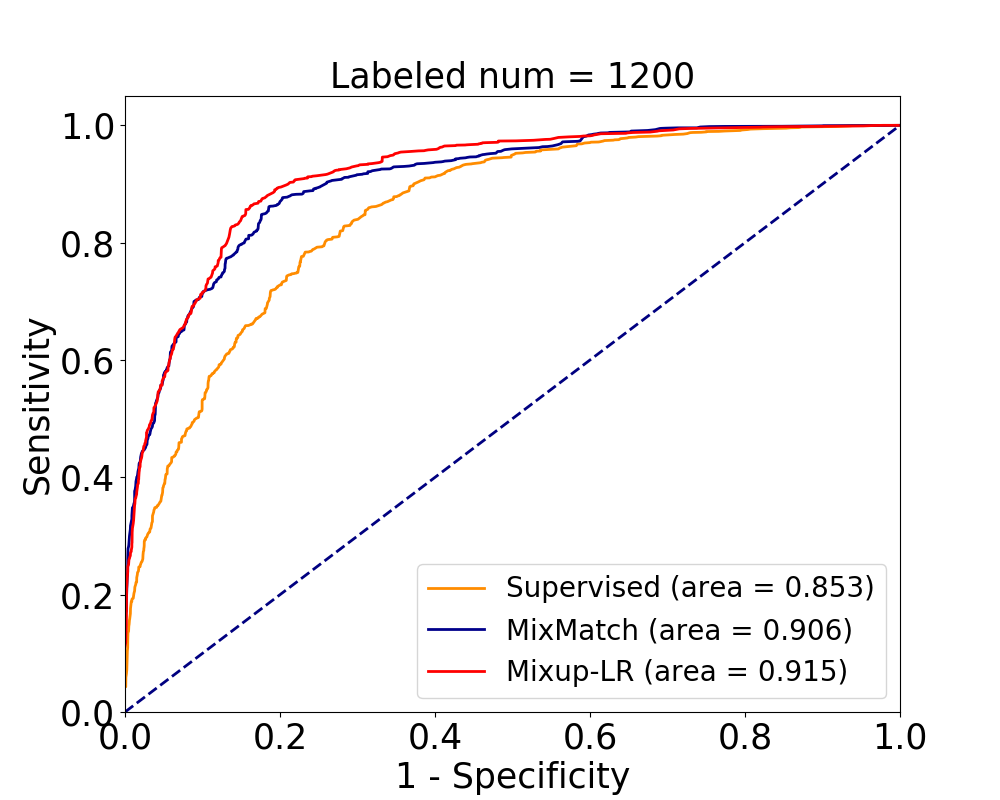}
\end{center}
\caption{ROC curves of the presented Mixup-LR compared to alternative models for classification of dermoscopic images into seven different disease categories when trained using 1200 labels.} 
\label{fig:roc_2}
\end{figure}

\subsubsection{Ablation study} 
Here, we primarily study the effect of $\zeta$ hyperparameter in the presented SSL method, using a labeled dataset of size 250 for the CIFAR-10 data set. The $\zeta$ hyperparameter controls the effect of the presented regularization with the mixup-based loss, as shown in Eq. (\ref{eq:final_loss}). 
We conducted this study for four different values (0, 1, 2, and 3), as shown in Table. \ref{tab:ablation}. While the inclusion of the Lipschitz penalty improves upon the baseline method (\emph{i.e.,} $\zeta$=0 \textit{vs.} rest), different hyperparameter values of $\zeta$ produce a similar result with the best case at $\zeta$=2. 
We also experiment to understand the effect of the quality of adversarial perturbation on the presented loss. 
Toward this, we increase the iteration number to 2 and for a single seed experiment, obtain the median accuracy as 88.42 on CIFAR-10 (250 labels) compared to 90.53 $\pm$ 0.99 with doing a single iteration. 
This shows that the single iteration for calculating adversarial perturbation for the Lipsthiz penalty is reasonable enough in these SSL problems corroborating with the previous discussion in the literature \cite{miyato2018virtual, terjek2019virtual}.
However, increasing the iteration beyond two may further improve the performance of the presented method.

\begin{table}[t]
\centering
    \begin{tabular}[t]{l|c}
    \hline
    Ablation & CIFAR-10 (250 labels)\\
    \hline
    $\zeta$ = 0 & 12.90 $\pm$ 2.10 \\
    $\zeta$ = 1 & 10.03 $\pm$ 0.63 \\
    $\zeta$ = 2 & \textbf{9.47} $\pm$ 0.99 \\
    $\zeta$ = 3 & 9.77 $\pm$ 0.58 \\
    \hline
    \end{tabular} 
    \caption{Ablation study on $\zeta$ hyperparameter.
    }
\label{tab:ablation}
\end{table}

\subsection{Robustness}
\label{sec:robustness}
The neural network vulnerability to the adversarial examples 
is well-known phenomena \cite{szegedy2013intriguing, goodfellow2014explaining}. We hypothesize that compared to the mixup-based approach, our presented Mixup-LR is more robust against the adversarial attacks for two reasons. First, Mixup-LR is Lipschitz regularized using adversarial samples, making them less insensitive to similar adversarial perturbations. Second, the Lipschitz penalty of Mixup-LR enforces local Lipschitzness in the classifier function. Recent works have theoretically demonstrated that regularizing the local Lipschitzness of the classifier function helps in achieving high clean and robust accuracy \cite{yang2020adversarial}. 
To test this hypothesis, we consider the Fast Gradient Sign Method (FGSM) \cite{goodfellow2014explaining}, which constructs adversarial examples in one single step. We consider two different pixel-wise perturbation amount (\emph{i.e.,} $\epsilon$) of 0.07 and 0.007. 
We evaluate this adversarial attack on the SSL model trained with the CIFAR-10 dataset.
In Table \ref{tab:adv}, we present the performance of networks trained with Mixup-LR compared against MixMatch. As we can see for all the models trained with different labeled numbers (250, 1000, and 4000), Mixup-LR demonstrates better robustness compared to MixMatch. 
However, note that the result of adversarial robustness might be biased since Mixup-LR was already better in terms of clean accuracy (as seen in Table \ref{tab:cifar10}). 
To investigate this, we evaluate the drop in performance after an adversarial attack. We present the result in Fig. \ref{fig:adv_drop} where the percentage drop of the presented Mixup-LR (solid line) is less than the corresponding MixMatch model (dotted line). This confirms that the presented Mixup-LR is robust compared to the model trained with the Mixup strategy only.

\begin{table}[t]
\centering
    \begin{tabular}[t]{l|c|c|c|c}
    \hline
    & \multicolumn{2}{c|}{MixMatch} & \multicolumn{2}{c}{Mixup-LR}\\
    $\epsilon$ & 0.007 & 0.07 & 0.007 & 0.07 \\
    \hline
    250 &  68.01 $\pm$ 10.03 & 75.26 $\pm$ 7.64 & 55.24 $\pm$ 4.79 & 63.67 $\pm$ 4.00\\
    1000 &  53.68 $\pm$ 2.35 & 63.67 $\pm$ 1.48 & 47.72 $\pm$ 3.04 & 57.64 $\pm$ 3.16\\
    4000 &  57.73 $\pm$ 1.80 & 67.55 $\pm$ 2.20 & 53.55 $\pm$ 6.08 & 61.69 $\pm$ 5.39\\
    \hline
    \end{tabular} 
    \caption{Accuracy of unlabeled data points on white-box FGSM adversarial examples on CIFAR-10. We compare the presented Mixup-LR against MixMatch for two different $\epsilon$ for three different labeled cases.}
\label{tab:adv}
\end{table}

\begin{figure}[t]
\begin{center}
\includegraphics[scale=0.20]{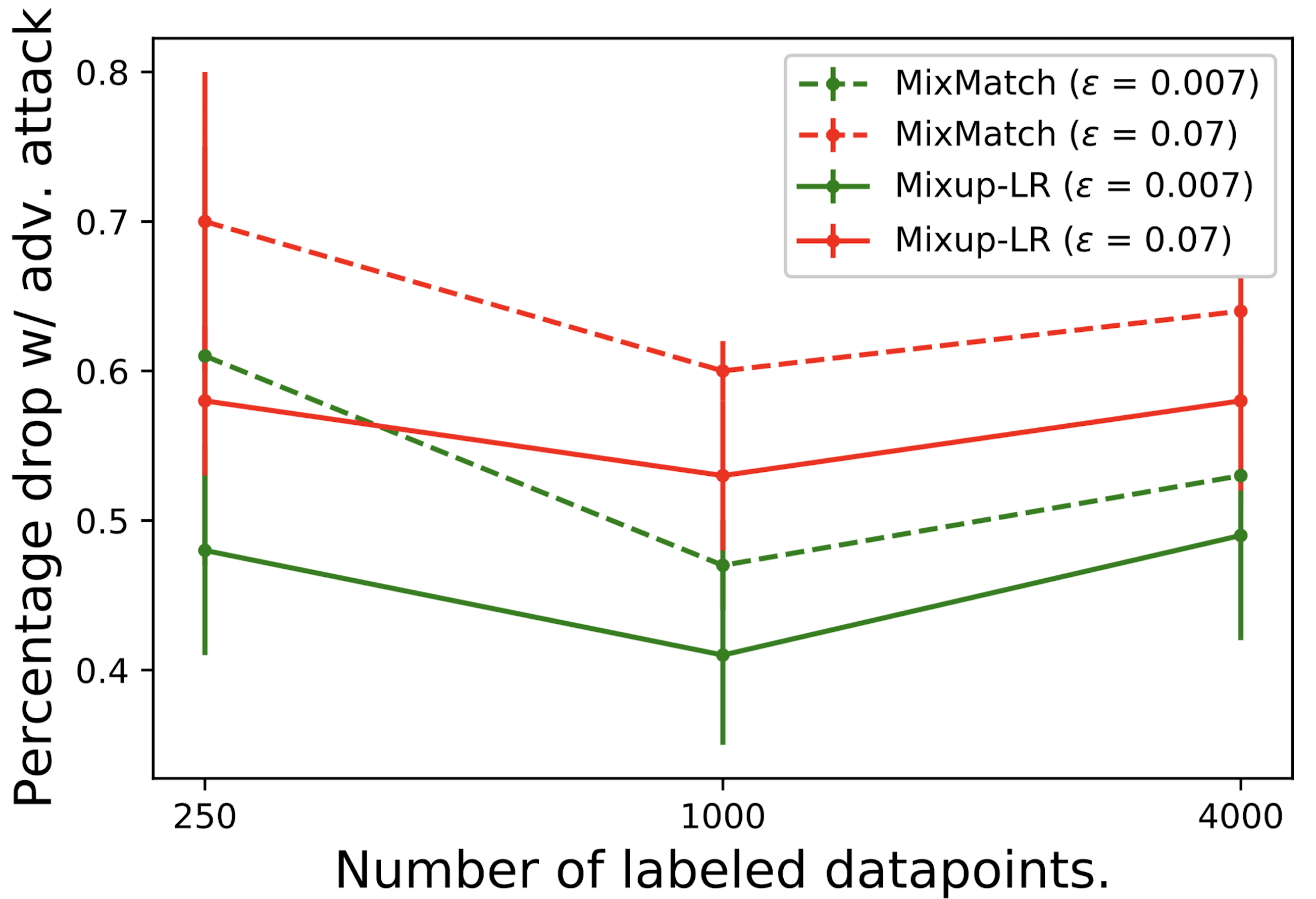}
\end{center}
\caption{Percentage drop for the presented Mixup-LR compared to MixMatch after adversarial attack (FGSM) for two perturbation amount $\epsilon = 0.07$ and $\epsilon = 0.007$.} 
\label{fig:adv_drop}
\end{figure}

\section{Discussion and Future Work}
While this research presents a novel SSL model (Mixup-LR) to improve the 
state-of-the-art MixMatch method on benchmark datasets, there are certain limitations of the current research which we discuss here.

First, the current approach to penalize the violation of the Lipschitz constraint might be expensive as it requires 1 step of back-propagation for each power iteration step while calculating the adversarial perturbation. Although this \textit{extra} computation is standard in Lipschitz regularization (e.g., Gradient Penalty \cite{gulrajani2017improved}), there are recent works that have demonstrated the efficacy of cheap techniques for obtaining adversarial examples \cite{shafahi2019adversarial}. As future work, we will consider such methods to eliminate the overhead cost of generating adversarial examples for Lipschitz regularization. 

Second, in this work, by augmenting the Lipschitz regularization with a mixup-based strategy, we control the Lipschitz constant of the deep neural network. 
As future work, we want to empirically validate this proposition by estimating the network's Lipschitz constants. Although estimation of the Lipschitz constant for deep networks often suffers from either lack of accuracy or poor scalability, some of the recent works have demonstrated the accurate performance \cite{fazlyab2019efficient}. We leave utilizing these latest findings of the literature for the future.

\section{Conclusion}
We presented a novel SSL method, Mixup-LR, which combines a mixup-based strategy with the explicit Lipschitz regularization. 
We first showed the effect of mixup regularization in promoting smoothness, where the mixup approach was found to bound only the Lipschitz constant of the gradient of the neural function. As such, we augmented mixup with explicit Lipschitz regularization to control the Lipschitz constant of the function itself. The efficacy of the presented Mixup-LR was demonstrated on three SSL benchmark data sets and one real-world clinical data set, through improvement over state-of-the-art MixMatch model along with other standard SSL algorithms. We also demonstrated the robustness of Mixup-LR against single-step adversarial attacks. 

\bibliographystyle{IEEEbib}
\bibliography{refs}
\end{document}